\documentclass{article}

\usepackage[accepted]{icml2019} 

\usepackage{amsmath,amsfonts,bm}

\def\1{\bm{1}}

\def\rr{{\textnormal{r}}}

\def\rvx{{\mathbf{x}}}

\def\ervx{{\textnormal{x}}}

\def\vzero{{\bm{0}}}
\def\vone{{\bm{1}}}
\def\vmu{{\bm{\mu}}}
\def\valpha{{\bm{\alpha}}}
\def\vlambda{{\bm{\lambda}}}
\def\vtheta{{\bm{\theta}}}
\def\va{{\bm{a}}}
\def\vb{{\bm{b}}}

\def\vr{{\bm{r}}}

\def\vv{{\bm{v}}}
\def\vw{{\bm{w}}}
\def\vx{{\bm{x}}}
\def\vy{{\bm{y}}}
\def\vz{{\bm{z}}}

\def\evalpha{{\alpha}}
\def\evbeta{{\beta}}

\def\evb{{b}}

\def\evy{{y}}
\def\evz{{z}}

\def\mG{{\bm{G}}}

\def\mI{{\bm{I}}}

\def\mK{{\bm{K}}}

\def\mU{{\bm{U}}}
\def\mV{{\bm{V}}}

\def\mZ{{\bm{Z}}}

\def\mSigma{{\bm{\Sigma}}}

\DeclareMathAlphabet{\mathsfit}{\encodingdefault}{\sfdefault}{m}{sl}
\SetMathAlphabet{\mathsfit}{bold}{\encodingdefault}{\sfdefault}{bx}{n}

\def\emG{{G}}

\def\emK{{K}}

\def\emU{{U}}

\newcommand{\pdata}{p_{\rm{data}}}

\DeclareMathOperator{\E}{\mathbb{E}}

\newcommand{\R}{\mathbb{R}}

\newcommand{\Var}{\mathrm{Var}}

\DeclareMathOperator*{\argmax}{arg\,max}
\DeclareMathOperator*{\argmin}{arg\,min}

\usepackage{microtype}
\usepackage{graphicx}
\usepackage{booktabs} 

\usepackage{amsmath}
\usepackage{amsthm}
\usepackage{thmtools}
\usepackage{thm-restate}
\usepackage{url}
\usepackage[ruled,vlined]{algorithm2e}    
\usepackage{subcaption}

\usepackage{hyperref}
\usepackage[capitalize,noabbrev]{cleveref}

\declaretheorem[name=Proposition]{prop}
\declaretheorem[sibling=prop]{lemma}

\newcommand{\x}{\vx}
\newcommand{\D}{\mathcal{D}}
\newcommand{\Dt}{\D_t}
\newcommand{\Dv}{\D_v}
\newcommand{\xn}{\x_{n}}
\newcommand{\zm}{\z_{m}}
\newcommand{\z}{\vz}
\newcommand{\y}{\vy}

\newcommand{\bphi}{\boldsymbol{\phi}}

\newcommand{\w}{\vw}

\newcommand{\bR}{\mathbb{\R}}
\newcommand{\ptilde}{\tilde{p}}
\newcommand{\cH}{\mathcal{H}}
\newcommand{\ud}{\mathrm{d}}
\newcommand{\tp}{^\mathsf{T}}

\newcommand{\A}{\approx}
\DeclareMathOperator{\N}{\mathcal{N}}
\DeclareMathOperator{\diag}{diag}

\renewcommand{\pdata}{p_0}  

\newcommand{\httpsurl}[1]{\href{https://#1}{\nolinkurl{#1}}}

\newcommand\numberthis{\addtocounter{equation}{1}\tag{\theequation}}

\newcommand{\mytitle}{Learning Deep Kernels for Exponential Family Densities}
\begin{document}

\twocolumn[
\icmltitle{\mytitle}



\icmlsetsymbol{equal}{*}

\begin{icmlauthorlist}
\icmlauthor{Li K.\ Wenliang}{equal,gatsby}
\icmlauthor{Danica J.\ Sutherland}{equal,gatsby}
\icmlauthor{Heiko Strathmann}{gatsby}
\icmlauthor{Arthur Gretton}{gatsby}
\end{icmlauthorlist}

\icmlaffiliation{gatsby}{Gatsby Computational Neuroscience Unit, University College London, London, U.K.}

\icmlcorrespondingauthor{Li K.\ Wenliang}{wenliang2012@gmail.com}
\icmlcorrespondingauthor{Danica J.\ Sutherland}{djs@djsutherland.ml}

\icmlkeywords{Machine Learning, ICML}

\vskip 0.3in
]



\printAffiliationsAndNotice{\icmlEqualContribution} 

\begin{abstract}
The kernel exponential family is a rich class of distributions, which can be fit efficiently and with statistical guarantees by score matching. Being required to choose \emph{a priori} a simple kernel such as the Gaussian, however, limits its practical applicability. We provide a scheme for learning a kernel parameterized by a deep network, which can find complex location-dependent features of the local data geometry. This gives a very rich class of density models, capable of fitting complex structures on moderate-dimensional problems. 
Compared to deep density models fit via maximum likelihood, our approach provides a complementary set of strengths and tradeoffs: in empirical studies, deep maximum-likelihood models can yield higher likelihoods, while our approach gives better estimates of the gradient of the log density, the \emph{score}, which describes the distribution's shape.
\end{abstract}

\begin{figure*}[t]  
  \includegraphics[width=\textwidth, trim=1cm 0 0 .75cm, clip]{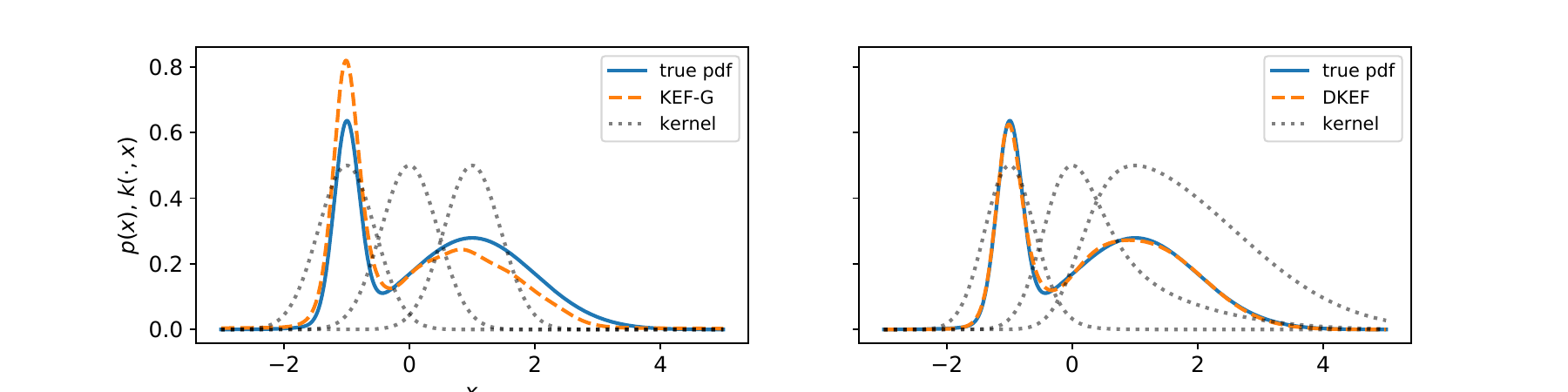}
    \caption{%
      Fitting few samples from a Gaussian mixture,
      using kernel exponential families.
      Black dotted lines show $k(-1, \x)$, $k(0, \x)$, and $k(1, \x)$.
      (Left) Using a location-invariant Gaussian kernel,
      the sharper component gets too much weight.
      (Right) A kernel parameterized by a neural network learns length scales that adapt to the density,
      giving a much better fit.
    }
    \label{fig:gaussian_scales}
\end{figure*}

\section{Introduction} \label{sec:intro}
Density estimation is a foundational problem in statistics and machine learning \citep{DevGyo85,Wasserman06},
lying at the core of both supervised and unsupervised machine learning problems.
Classical techniques such as kernel density estimation, however,
struggle to exploit the structure inherent to complex problems,
and thus can require unreasonably large sample sizes for adequate fits.
For instance, assuming only twice-differentiable densities,
the $L_2$ risk of density estimation with $N$ samples in $D$ dimensions
scales as $\mathcal{O}(N^{-4/(4+D)})$ \citep[Section 6.5]{Wasserman06}.

One promising approach for incorporating this necessary structure is
the \emph{kernel exponential family} \citep{kexpfam,Fukumizu-09a,SriFukGre17}.
This model allows for any log-density which is suitably smooth under a given kernel,
i.e.\ any function in the corresponding reproducing kernel Hilbert space. 
Choosing a finite-dimensional kernel recovers any classical exponential family,
but when the kernel is sufficiently powerful the class becomes very rich:
dense in the family of continuous probability densities on compact
domains in KL, TV, Hellinger,
and $L^r$ distances
\citep[Corollary 2]{SriFukGre17}.
The normalization constant is not available in closed form,
making fitting by maximum likelihood difficult,
but the alternative technique of \emph{score matching} \citep{Hyv05}
allows for practical usage with theoretical convergence guarantees \citep{SriFukGre17}.

The choice of kernel directly corresponds to a smoothness assumption on the model,
allowing one to design a kernel corresponding to prior knowledge about the target density.
Yet explicitly deciding upon a kernel to incorporate that knowledge can be complicated.
Indeed, previous applications of the kernel exponential family model
have exclusively employed simple kernels, such as the Gaussian,
with a small number of parameters (e.g.\ the length scale) chosen by heuristics
or cross-validation \citep[e.g.][]{SasHyvSug14,StrSejLiv15}.
These kernels are typically {\em spatially invariant}, corresponding to a uniform smoothness assumption across the domain. 
Although such kernels are sufficient for consistency
in the infinite-sample limit,
the induced models can fail in practice on finite datasets, 
especially if the data takes differently-scaled shapes in different parts of the space.
\Cref{fig:gaussian_scales} (left) illustrates this problem when fitting a simple mixture of Gaussians.
Here there are two ``correct'' bandwidths, one for the broad mode and one for the narrow mode.
A translation-invariant kernel  must pick a single one, e.g. an average between the two,
and any choice will yield a poor fit on at least part of the density.

In this work, we propose to \emph{learn} the kernel of an exponential family directly from data.
We can then achieve far more than simply tuning a length scale,
instead learning location-dependent kernels that adapt to the underlying shape and smoothness of the target density.
We use kernels of the form 
\begin{equation}
k(\x, \y) = \kappa(\bphi(\x), \bphi(\y))
\label{eq:deep-kernel}
,\end{equation}
where the deep network $\bphi$ extracts features of the input
and $\kappa$ is a simple kernel (e.g.\ a Gaussian) on those features.
These types of kernels have seen success in supervised learning \citep{wilson2016deep,JeaXieErm18}
and critic functions for implicit generative models \citep{mmd-gan,cramer-gan,demystifying-mmd-gans,ArbSutBinGre18},
among other settings.
We call the resulting model a deep kernel exponential family (DKEF).

We can train
both kernel parameters (including all the weights of the deep network) and, unusually, even \emph{regularization} parameters
directly on the data,
in a form of meta-learning.
Normally, directly optimizing regularization parameters would always yield 0,
since their beneficial effect in preventing overfitting is by definition not seen on the training set.
Here, though, we can exploit the closed-form fit of the kernel exponential family to optimize a ``held-out'' score (\cref{sec:fitting}). 
\Cref{fig:gaussian_scales} (right) demonstrates the success of this model on the same mixture of Gaussians;
here the learned, location-dependent kernel gives a much better fit.

We compare the results of our new model to recent general-purpose deep density estimators,
primarily
\emph{autoregressive models} \citep{rnade,made,pixelrnn} and  \emph{normalizing flows} \citep{norm-flows,nvp,maf}.
These models learn deep networks with structures designed to compute normalized densities,
and are fit via maximum likelihood.
We explore the strengths and limitations of both deep likelihood models and deep kernel exponential families
on a variety of datasets,
including artificial data designed to illustrate scenarios where certain surprising problems arise,
as well as benchmark datasets used previously in the literature.
The models fit by maximum likelihood typically give somewhat higher likelihoods, 
whereas the deep kernel exponential family generally better fits the \emph{shape} of the distribution.


\section{Background} \label{sec:background}
\paragraph{Score matching} 
Suppose we observe $\mathcal{D}=\{\xn\}_{n=1}^N$, 
a set of independent samples $\xn \in \bR^D$
from an unknown density $\pdata(\x)$.
We posit a class of possible models $\{ p_\vtheta \}$,
parameterized by $\vtheta$;
our goal is to use the data $\mathcal D$ to select some $\hat\vtheta$
such that $p_{\hat\vtheta} \approx \pdata$.
The standard approach for selecting $\vtheta$ is maximum likelihood:
$\hat\vtheta = \argmax_\vtheta \prod_{n=1}^N p_\vtheta(\xn)$.

Many interesting model classes, however, are defined as
$p_\vtheta(\x) = \tilde p_\vtheta(\x) / Z_\vtheta$,
where  
the normalization constant $Z_\vtheta = \int_\x \ptilde_\vtheta(\x) \ud\x$
cannot be easily computed.
In this setting,
an optimization algorithm to estimate $\vtheta$ by maximum likelihood
requires estimating (the derivative of) $Z_\vtheta$
for each candidate $\vtheta$ considered during optimization.
Moreover, the maximum likelihood solution may not even be well-defined when $\vtheta$ is infinite-dimensional
\citep{Barron-91,Fukumizu-09a}.
The intractability of maximum likelihood led \citet{Hyv05} to propose an alternative objective, called \emph{score matching}.
Rather than maximizing the likelihood, one minimizes the Fisher divergence
$J(p_\vtheta \| \pdata)$:
\begin{equation}
    \frac{1}{2}\int \pdata(\x) \|\nabla_\x \log p_\vtheta(\x) - \nabla_\x \log \pdata(\x)\|_2^2 \,\ud\x
\label{eq:true_score}
.\end{equation}
Under mild regularity conditions, this is equal to
\begin{equation}
    \!\!\!\!  
    \int_\x \pdata(\x) \sum_{d=1}^D \left[
      \partial_d^2\log p_\vtheta(\x) 
    + \frac{1}{2} (\partial_d \log p_\vtheta(\x))^2
    \right] \ud\x
\label{eq:estimate_score}
,\end{equation}
up to an additive constant depending only on $\pdata$,
which can be ignored during training.
Here $\partial_d^n f(\x)$ denotes $\frac{\partial^n}{(\partial \evy_d)^n} f(\y) \rvert_{\y = \x}$.
We can estimate \eqref{eq:estimate_score} with $\hat J(p_\vtheta, \D)$:
\begin{equation} \label{eq:score-estimator}
  \frac1N \sum_{n=1}^N \sum_{d=1}^D \left[
      \partial_d^2\log p_\vtheta(\xn) 
    + \frac{1}{2} (\partial_d \log p_\vtheta(\xn))^2
    \right]
.\end{equation}
Notably, \eqref{eq:score-estimator} does not depend on $Z_\vtheta$,
and so we can minimize it to find 
an unnormalized model $\ptilde_{\hat\vtheta}$ for $\pdata$.
Score matching is consistent in the well-specified setting \citep[Theorem 2]{Hyv05},
and was related to maximum likelihood by \citet{Lyu09},
who argues it finds a fit robust to infinitesimal noise.

Unnormalized models $\ptilde$ are sufficient for many tasks \citep{LecChoHad06},
including finding modes,
approximating Hamiltonian Monte Carlo on targets without gradients \citep{StrSejLiv15},
and learning discriminative features \citep{score-features}.
If we require a normalized model, however,
we can estimate the normalizing constant once,
after estimating $\vtheta$;
this will be far more computationally efficient    
than estimating it at each step of an iterative maximum likelihood optimization algorithm.

\paragraph{Kernel exponential families}
The kernel exponential family \citep{kexpfam,SriFukGre17}
is the class of all densities
satisfying a smoothness constraint:
$\log \ptilde(\vx) = f(\vx) + \log q_0(\vx)$,
where $q_0$ is some fixed function
and $f$ is any function 
in the reproducing kernel Hilbert space $\cH$ with kernel $k$.
This class is an exponential family
with natural parameter $f$ and sufficient statistic $k(\vx, \cdot)$,
due to the reproducing property
$f(\vx) = \langle f, k(\vx, \cdot) \rangle_\cH$:
\begin{equation*} 
  \ptilde_f(\x)
  = \exp\left( f(\x) \right) q_0(\x)
  = \exp\left( \langle f, k(\x, \cdot) \rangle_\cH \right) q_0(\x)
.\end{equation*}
Using a simple finite-dimensional $\cH$,
we can recover any standard exponential family,
e.g.\ normal, gamma, or Poisson; 
if $\cH$ is sufficiently rich,
the family 
can approximate any continuous distribution with tails like $q_0$ arbitrarily well
\citep[Example 1 and Corollary 2]{SriFukGre17}.

These models do not in general have a closed-form normalizer.
For some $f$ and $q_0$, $\ptilde_f$ may not even be normalizable,
but if $q_0$ is e.g.\ a Gaussian density,
typical choices of $\bphi$ and $\kappa$ in \eqref{eq:deep-kernel}
guarantee a normalizer exists
(
\cref{sec:normalizable}).

\citet{SriFukGre17} proved good statistical properties for choosing $f \in \cH$
by minimizing a regularized form of \eqref{eq:score-estimator},
$\hat f = \argmin_{f \in \cH} \hat J(\tilde p_f, \D) + \lambda \lVert f \rVert_\cH^2$,
but their algorithm has an impractical computational cost of $\mathcal O(N^3 D^3)$.
This can be alleviated with the Nystr\"om-type ``lite'' approximation \citep{StrSejLiv15,SutStrArb18}:
select $M$ inducing points $\z_m \in \R^D$,
and select $f \in \cH$ as
\begin{equation} \label{eq:f_lite}
  f_{\valpha,\vz}^{k}(\x) = \sum_{m=1}^M \evalpha_m k(\x, \z_m)
  ,\quad
  \ptilde_{\valpha,\z}^k
  = \ptilde_{f_{\valpha,\z}^k}
.\end{equation}
As the span of $\{ k(\z, \cdot) \}_{\z \in \bR^D}$ is dense in $\cH$,
this is a natural approximation,
similar to classical RBF networks \citep{rbf-nets}.
The ``lite'' model often yields excellent empirical results at much lower computational cost than the full estimator.
We can regularize \eqref{eq:score-estimator} in several ways
and still find a closed-form solution for $\valpha$.
In this work, our loss
$\hat{J}(f_{\valpha,\z}^k, \vlambda, \D)$
will be
\begin{equation*}
    \hat{J}(\ptilde_{\valpha,\vz}^k, \D)
     + \frac{\lambda_\alpha}{2} \lVert \valpha \rVert^2
     + \frac{\lambda_C}{2 N} \sum_{n=1}^N \sum_{d=1}^D \left[ \partial_d^2 \log \ptilde_{\valpha,\z}^k(\xn) \right]^2
.\end{equation*}
\citet{SutStrArb18} used a small $\lambda_\alpha$ for numerical stability
but primarily regularized with $\lambda_H \lVert f_{\valpha,\z}^k \rVert_\cH^2$.
As we change $k$, however, $\lVert f \rVert_\cH$ changes meaning,
and we found empirically that this regularizer tends to harm the fit. 
The $\lambda_C$ term was recommended by \citet{KinLec10},
encouraging the learned log-density to be smooth without much extra computation;
it provides some empirical benefit in our context.
Given $k$, $\z$, and $\vlambda$,
\cref{thm:alpha} (\cref{sec:alpha-proof}) shows
we can find the optimal $\valpha$ by solving an $M \times M$ linear system
in $\mathcal O(M^2 N D + M^3)$ time:
the $\valpha$ which minimizes
$\hat J(f_{\valpha,\vz}^k, \vlambda, \D)$
is
\begingroup \allowdisplaybreaks
\begin{gather*}
  \valpha(\vlambda, k, \vz, \D)
   = -\left(
    \mG + \lambda_\alpha \mI 
    + \lambda_C \mU
  \right)^{-1} \vb
\numberthis \label{eq:alpha_lite}
\\
    \emG_{m,m'}
     = \frac1N \sum_{n=1}^N \sum_{d=1}^D 
          \partial_d k(\xn, \zm) \, \partial_d k(\xn, \z_{m'})
\\
    \emU_{m,m'}
     = \frac1N \sum_{n=1}^N \sum_{d=1}^D
          \partial_d^2 k(\xn, \zm) \, \partial_d^2 k(\xn, \z_{m'})
\\
    \evb_{m}
     = \frac1N \sum_{n=1}^N \sum_{d=1}^D \partial^2_d k(\xn, \zm)
    			  + \partial_d \log q_0(\xn) \, \partial_d k(\xn, \zm)
\\ \qquad\qquad\qquad
    			  + \lambda_C \partial_d^2 \log q_0(\xn) \, \partial_d^2 k(\xn, \zm)
.\end{gather*}
\endgroup

\section{Fitting Deep Kernels}\label{sec:fitting}

All previous applications of score matching in the kernel exponential family of which we are aware
\citep[e.g.][]{SriFukGre17,StrSejLiv15,kexpfam-structure,SutStrArb18}
have used kernels of the form
$k(\vx, \vy) = \exp\left( -\frac{1}{2 \sigma^2} \lVert \vx - \vy \rVert^2 \right)
         + r \left( \vx\tp \vy + c \right)^2$,
with kernel parameters and regularization weights either fixed a priori or selected via cross-validation.
This simple form allows the various kernel derivatives
required in \eqref{eq:alpha_lite} to be easily computed by hand,
and the small number of parameters makes grid search adequate for model selection.
But, as discussed in \cref{sec:intro},
these simple kernels are insufficient for complex datasets.
Thus we wish to use a richer class of kernels $\{ k_\vw \}$,
with a large number of parameters $\vw$~--~in particular, kernels defined by a neural network.
This prohibits model selection via simple grid search.

One could attempt to directly minimize $\hat J(f_{\valpha,\z}^{k_\w}, \vlambda, \D)$
jointly in the kernel parameters $\vw$, the model parameters $\valpha$, and perhaps the inducing points $\vz$.
Consider, however, the case where we simply use a Gaussian kernel and $\{ \z_m \} = \D$.
Then we can achieve arbitrarily good values of \eqref{eq:estimate_score}
by taking $\sigma \to 0$,
drastically overfitting to the training set $\D$.

We can avoid this problem
~--~and additionally find the best values for the regularization weights $\vlambda$~--~%
with a form of meta-learning.
We find choices for the kernel and regularization
which will give us a good value of $\hat J$ on a ``validation set'' $\Dv$
when fit to a fresh ``training set'' $\Dt$.
Specifically,
we take stochastic gradient steps following
$\nabla_{\lambda,\vw,\z} \hat J(\ptilde_{\valpha(\lambda, k_\vw, \z, \Dt), \z}^{k_\vw}, \Dv)$.
We can easily do this because we have a differentiable closed-form expression \eqref{eq:alpha_lite} for the fit to $\Dt$,
rather than having to e.g.\ back-propagate through an unrolled iterative optimization procedure.
As we used small minibatches in this procedure,
for the final fit we use the whole dataset:
we first freeze $\vw$ and $\vz$ and find the optimal $\vlambda$ for the whole training data,
then finally fit $\valpha$ with the new $\vlambda$.
This process is summarized in \cref{alg}.

\begin{algorithm}[tb]
\SetAlgoLined
\SetKwInput{Input}{input}
\SetKwInput{Output}{return}
\Input{%
  Dataset $\D$;
  initial inducing points $\z$,
  kernel parameters $\vw$,
  regularization $\vlambda = (\lambda_\alpha, \lambda_C)$
}
Split $\D$ into $\D_1$ and $\D_2$\;
\emph{Optimize $\vw$, $\vlambda$, $\z$, and maybe $q_0$ params:}\DontPrintSemicolon\;\PrintSemicolon
\While{$\hat J(\ptilde_{\valpha(\vlambda, k_\vw, \z, \D_1),\z}^{k_\vw}, \D_2)$ still improving}{
    Sample disjoint data subsets $\Dt, \Dv \subset \D_1$\;
    $f(\cdot) = \sum_{m=1}^M \evalpha_m(\vlambda, k_\vw, \z, \Dt) k_\vw(\zm, \cdot)$\;
    $\hat J\!=\! \frac{1}{|\Dv|} \! \sum_{n=1}^{|\Dv|} \sum_{d=1}^D \left[ \partial_d^2f(\xn) + 
        \frac{1}{2} (\partial_df(\xn))^2 \right]$\;
    Take SGD step in $\hat J$ for $\vw$, $\vlambda$, $\z$, maybe $q_0$\ params;
}
\emph{Optimize $\vlambda$ for fitting on larger batches:}\DontPrintSemicolon\;\PrintSemicolon
\While{$\hat J(\ptilde_{\valpha(\vlambda, k_\vw, \z, \D_1),\z}^{k_\vw}, \D_2)$ still improving}{
    $f(\cdot) = \sum_{m=1}^M \evalpha_m(\vlambda, k_\vw, \z, \D_1) k_\vw(\cdot, \zm)$\;
    Sample subset $\D_v \subset D_2$\;
    $\hat J \! = \! \frac{1}{|\D_v|} \! \sum_{n=1}^{|\D_v|} \sum_{d=1}^D \left[ \partial_d^2f(\xn) + 
        \frac{1}{2} (\partial_df(\xn))^2 \right]$\;
    Take SGD steps in $\hat J$ for $\vlambda$ only\;
}
\emph{Finalize $\valpha$ on $\D_1$:}\DontPrintSemicolon\;\PrintSemicolon
Find $\valpha = \valpha(\vlambda, k_\vw, \vz, \D_1)$\;
\Output{$\log \ptilde(\cdot) = \sum_{m=1}^M \evalpha_m k_\vw(\cdot, \zm) + \log q_0(\cdot)$;}
\caption{Full training procedure}
\label{alg}
\end{algorithm}

\paragraph{Computing kernel derivatives} 
Solving for $\valpha$ and computing the loss \eqref{eq:score-estimator}
require matrices of kernel second derivatives,
but current deep learning-oriented automatic differentiation systems
are not optimized for evaluating 
tensor-valued 
higher-order derivatives at once. We therefore implement backpropagation to compute $\mG$, $\mU$, and $\vb$ of \eqref{eq:alpha_lite} as TensorFlow operations \citep{tensorflow} to obtain the scalar loss $\hat{J}$,
and used TensorFlow's automatic differentiation only to optimize $\vw$, $\vz$, $\vlambda$, and $q_0$ parameters.

Backpropagation to find these second derivatives requires explicitly computing the Hessians of 
intermediate layers of the network, which becomes quite expensive as the model grows;
this limits the size of kernels that our model can use. 
A more efficient implementation based on Hessian-vector products
might be possible in an automatic differentiation system with better support for matrix derivatives.


\paragraph{Kernel architecture}
We will choose our kernel $k_\vw(\x, \y)$ as
a mixture of $R$ Gaussian kernels with length scales $\sigma_r$,
taking in features of the data extracted by a network $\bphi_{\w_r}(\cdot)$:
\begin{equation} \label{eq:k_theta}
  \sum_{r=1}^R \rho_r\exp\left(
    -\frac{1}{2 \sigma_r^2} \left\lVert \bphi_{\w_r}(\x) - \bphi_{\w_r}(\y) \right\rVert^2
    \right)
.\end{equation}
Combining $R$ components makes it easier to account for both short-range and long-range dependencies.
We constrain $\rho_r \ge 0$ to ensure a valid kernel,
and $\sum_{r=1}^R \rho_r = 1$ for simplicity.
The networks $\bphi_{\w}$ are made of $L$ fully connected layers of width $W$.
For $L > 1$, we found that adding a skip connection from data directly to the top layer speeds up learning. 
A softplus nonlinearity, $\log(1 + \exp(x))$,
ensures that the model is twice-differentiable so \eqref{eq:estimate_score} is well-defined.

\subsection{Behavior on Mixtures} \label{sec:mixtures}
One interesting limitation of score matching is the following:
suppose that $\pdata$ is composed of two disconnected components,
$\pdata(\x) = \pi p_1(\x) + (1 - \pi) p_2(\x)$
for $\pi \in (0, 1)$ and $p_1$, $p_2$ having disjoint, separated support.
Then
$\nabla \log \pdata(x)$ will be $\nabla \log p_1(\x)$ in the support of $p_1$,
and $\nabla \log p_2(\x)$ in the support of $p_2$.
Score matching
compares $\nabla \log \ptilde_\vtheta$ to $\nabla\log p_1$ and $\nabla\log p_2$,
but is completely blind to $\ptilde_\vtheta$'s relative mass between the two components;
it is equally happy with \emph{any} reweighting of the components.

If all modes are connected by regions of positive density,
then the log density gradient in between components will determine their relative weight,
and indeed score matching is then consistent.
But when $\pdata$ is \emph{nearly} zero between two dense components,
so that there are no or few samples in between,
score matching will generally have insufficient evidence to weight nearly-separate components.

\Cref{thm:mixtures} (\cref{sec:mixture-details}) studies the the kernel exponential family in this case.
For two components that are completely separated according to $k$,
\eqref{eq:alpha_lite} fits each as it would if given only that component,
except that the effective $\lambda_\alpha$ is scaled:
smaller components are regularized more.

\Cref{sec:fit-gaussians} studies a simplified case
where the kernel length scale $\sigma$ is far wider than the component;
then the density ratio between components, which should be $\frac{\pi}{1 - \pi}$,
is approximately $\exp\left( \frac{D}{2 \sigma^2 \lambda_\alpha} \left(\pi - \frac12 \right) \right)$.
Depending on the value of $\frac{D}{2 \sigma^2 \lambda_\alpha}$,
this ratio will often either be quite extreme, or nearly $1$.
It is unclear, however, how well this result generalizes to other settings.

A heuristic workaround when disjoint components are suspected is as follows:
run a clustering algorithm to identify disjoint components,
separately fit a model to each cluster,
then weight each model according to its sample count.
When the components are well-separated, this clustering is straightforward,
but it may be difficult in high-dimensional cases when samples are sparse but not fully separated.

\subsection{Model Evaluation} \label{sec:model-eval}
In addition to qualitatively evaluating fits,
we will evaluate our models with three quantitative criteria.
The first is the finite-set Stein discrepancy \citep[FSSD;][]{kstein-linear},
a measure of model fit which does not depend on the normalizer $Z_\vtheta$.
It
examines the fit of the model at $J$ test locations $\mV = \{ \vv_b \}_{b=1}^B$
using an kernel $l(\cdot, \cdot)$,
as
$
\frac{1}{D B} \sum_{b=1}^B \lVert
\E_{\x \sim \pdata}[l(\x, \vv_b) \nabla_\x \log p(\x) + \nabla_\x l(\x, \vv_b)]
\rVert^2$.
With randomly selected $\mV$ and some mild assumptions,
it is zero if and only if $p = \pdata$. 
We use a Gaussian kernel with bandwidth equal to the median distance between test points,\footnote{%
  This reasonable choice avoids tuning any parameters.
  We do not optimize the kernel or the test locations
  to avoid a situation in which model $p$ is better than $p'$ in some respects but $p'$ better than $p$ in others;
  instead, we use a simple default mode of comparison.
}
and choose $\mV$ by adding small Gaussian noise to data points.
\citet{kstein-rel} construct a hypothesis test
to test which of $p$ and $p'$ is closer to $\pdata$ in the FSSD.
We will report a score~--~the $p$-value of this test~--~%
which is near 0 when model $p$ is better, near 1 when model $p'$ is better,
and around $\frac12$ when the two models are equivalent.
We emphasize that we are using this as a model comparison score on an interpretable scale,
but \emph{not} following a hypothesis testing framework. Another similar performance measure
is the kernel Stein discrepancy (KSD) \citep{ksd}, where the model's goodness-of-fit is evaluated 
at all test data points rather than at random test locations.
We omit the results
as they are essentially identical to that of the FSSD,
even across a wide range of kernel bandwidths.

As all our models are twice-differentiable,
we also compare the score-matching loss \eqref{eq:score-estimator} on held-out test data.
A lower score-matching loss implies a smaller Fisher divergence between model and data distributions.

Finally, we compare test log-likelihoods,
using importance sampling estimates of the normalizing constant $Z_\vtheta$:
\[
  \hat Z_\vtheta
  = \frac{1}{U} \sum_{u=1}^U \rr_u
  \;\text{ where }\;
  \rr_u := \frac{\ptilde_\vtheta(\y_u)}{q_0(\y_u)}
  ,\;
  \y_u \sim q_0
,\]
so
$
  \E \hat Z_\vtheta
  = \int \frac{\ptilde_\vtheta(\y_u)}{q_0(\y_u)} q_0(\y_u)
  = Z_\vtheta
$.
Our log-likelihood estimate is
$\log \hat p_\vtheta(\x) = \log \tilde p_\vtheta(\x) - \log \hat Z_\vtheta$.
This estimator is consistent,
but Jensen's inequality tells us that
$\E \log\hat p_\vtheta(\x) > \log p_\vtheta(\x)$,
so our evaluation will be over-optimistic.
Worse, the variance of $\log \hat Z_\vtheta$ can be misleadingly small when the bias is still quite large; we observed this in our experiments.    
We can, though, bound the bias:
\begin{restatable}{prop}{biasestthm} \label{thm:bias-est}
  Suppose that $a, s \in \bR$ are such that
  $\Pr(\rr_u \ge a) = 1$
  and $\Pr(\rr_u \le s) \le \rho < \frac12$.
  Define
  $t := (s + a) / 2$,
  $\psi(q, Z_\vtheta) := \log \frac{Z}{q} + \frac{q}{Z} - 1$,
  and let $P := \max\left( \psi(a, Z_\vtheta), \psi(t, Z_\vtheta) \right)$.
  Then
  \[
    \log Z_\vtheta - \E \log \hat Z_\vtheta
    \le
    \frac{\psi\left( t, Z_\vtheta \right)}{\left( Z_\vtheta - t \right)^2} \frac{\Var[\rr_u]}{U}
    +
    P
    \left( 4 \rho (1 - \rho) \right)^{\frac{U}{2}}
  .\]
\end{restatable}
(Proof in \cref{sec:bias-proofs}.)
We can find $a$, $s$ because we propose from $q_0$,
and thus we can effectively estimate the bound (\cref{sec:bias-estimator}).
This estimate of the upper bound is itself biased upwards (\cref{thm:bias-est-bias}),
so it is likely, though not guaranteed, that the estimate overstates the amount of bias.


\subsection{Previous Attempts at Deep Score Matching}
\citet{KinLec10} used score matching to train a (one-layer) network to output an unnormalized log-density.
This approach is essentially a special case of ours:
use the kernel $k_\vw(\x, \y) = \phi_\vw(\x) \phi_\vw(\y)$,
where $\phi_\vw : \R^D \to \R$.
Then the function $f_{\valpha,\z}^{k_\vw}(\x)$ from \eqref{eq:f_lite} is
\[
  \sum_{m=1}^M \alpha_m \phi_\vw(\zm) \phi_\vw(\x)
  = \left[ \sum_{m=1}^M \alpha_m \phi_\vw(\zm) \right] \phi_\vw(\x)
.\]
The scalar in brackets is fit analytically, so
$\log p$ is determined almost entirely by the network $\phi_\vw$
plus $\log q_0(\x)$.

\citet{deen} recently also attempted parameterizing the unnormalizing log-density as a deep network,
using an approximation called Parzen score matching.
This approximation requires a global constant bandwidth to define the Parzen window size for the loss function,
fit to the dataset before learning the model.
This is likely only sensible on datasets for which simple fixed-bandwidth kernel density estimation is appropriate;
on more complex datasets, the loss may be poorly motivated.
It also leads to substantial oversmoothing visible in their results.
As they did not provide code for their method,
we do not compare to it empirically.

\section{Experimental Results}\label{sec:results}

In our experiments,
we compare to several alternative methods.
The first group are all fit by maximum likelihood,
and broadly fall into (at least) one of two categories:
\emph{autoregressive models}
decompose $p(\ervx_1,\dots,\ervx_D) = \prod_{d=1}^D p(\ervx_d|\rvx_{\le d})$
and learn a parametric density model for each of these conditionals.
\emph{Normalizing flows} 
instead apply a series of invertible transformations to some simple initial density, say standard normal,
and then compute the density of the overall model via the Jacobian of the transformation.
We use implementations\footnote{\httpsurl{github.com/gpapamak/maf}} of the following several models in these categories by \citet{maf}:

\textbf{MADE}~\citep{made} masks the connections of an autoencoder so it is autoregressive.
We use two hidden layers, and each conditional a Gaussian. 
\textbf{MADE-MOG} is the same but with each conditional a mixture of 10 Gaussians.

\textbf{Real NVP}~\citep{nvp} is a normalizing flow; we use a general-purpose form for non-image datasets.

\textbf{MAF}~\citep{maf}.
A combination of a normalizing flow and MADE, where the base density is modeled by MADE, 
with 5 autoregressive layers.
\textbf{MAF-MOG} instead models the base density by MADE-MOG.

For the models above, we use layers of width 30 for experiments on synthetic data, and 100 for benchmark datasets.
Larger values did not improve performance.

\textbf{KCEF}~\citep{kcef}. Inspired by autoregressive models, the density is modeled by a cascade of kernel conditional exponential family distributions, fit by score matching with Gaussian kernels.\footnote{\httpsurl{github.com/MichaelArbel/KCEF}}

\textbf{DKEF}.
On synthetic datasets, we consider four variants of our model with one kernel component, $R=1$.
KEF-G refers to the model using a Gaussian kernel with a learned bandwidth. 
DKEF-G-15 has the kernel \eqref{eq:k_theta}, with $L = 3$ layers of width $W = 15$.
DKEF-G-50 is the same with $W = 50$.
To investigate whether the top Gaussian kernel helps performance,
we also train DKEF-L-50, whose kernel is $k_\theta(\x,\y)=\bphi_\w(\x)\cdot\bphi_\w(\y)$, where $\bphi_\w$ has $W = 50$.
To compare with the architecture of \citet{KinLec10}, DKEF-L-50-1 has
the same architecture as DKEF-L-50 except that we add an extra layer with a single neuron, and use $M = 1$.
In all experiments,
$q_0(\x) = \prod_{d=1}^D \exp\left( - \lvert x_d - \mu_d \rvert^{\beta_d} / (2 \sigma_d^2) \right)$,
with $\beta_d > 1$.
On benchmark datasets, we use DKEF-G-50 and KEF-G with three kernel components, $R=3$.
Code for DKEF is at {\httpsurl{github.com/kevin-w-li/deep-kexpfam}}. 

\subsection{Behavior on Synthetic Datasets}
We first demonstrate the behavior of the models on several two-dimensional synthetic datasets
Funnel, Banana, Ring, Square, Cosine, Mixture of Gaussians (MoG) and 
Mixture of Rings (MoR). Together, they cover a range of geometric complexities and
multimodality.

\begin{figure*}
	\includegraphics[height=0.86\textheight, trim=1.0cm 1cm 0.2cm 0.3cm, clip]{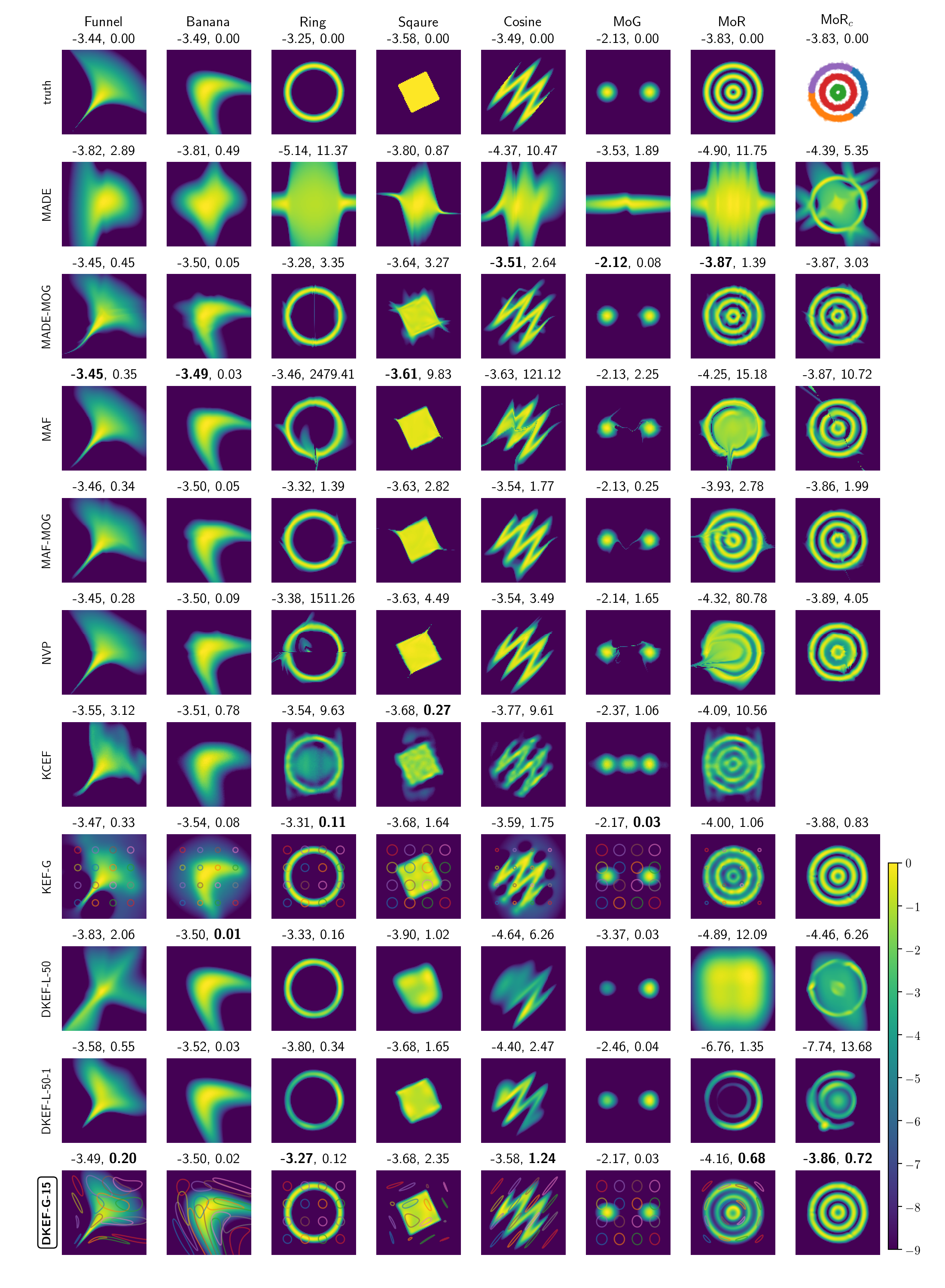}
    \centering
    \caption{Log densities learned by different models. Our model is DKEF-G-15 at the bottom row.
    Columns are different synthetic datasets. The rightmost columns shows a 
    mixture of each model (except KCEF) on the same clustering of MoR.
    We subtracted the maximum from each log density, and clipped the minimum value at $-9$.
    Above each panel are shown the average log-likelihoods (left) and Fisher divergence (right) on held-out data points.
    Bold indicates the best fit.
    For DKEF-G models, 
    faint colored lines correspond to contours at 0.9 of the kernel evaluated at different locations.
    }
    \label{fig:toys}
\end{figure*}   

\begin{figure*}[t]
  \includegraphics[width=\textwidth]{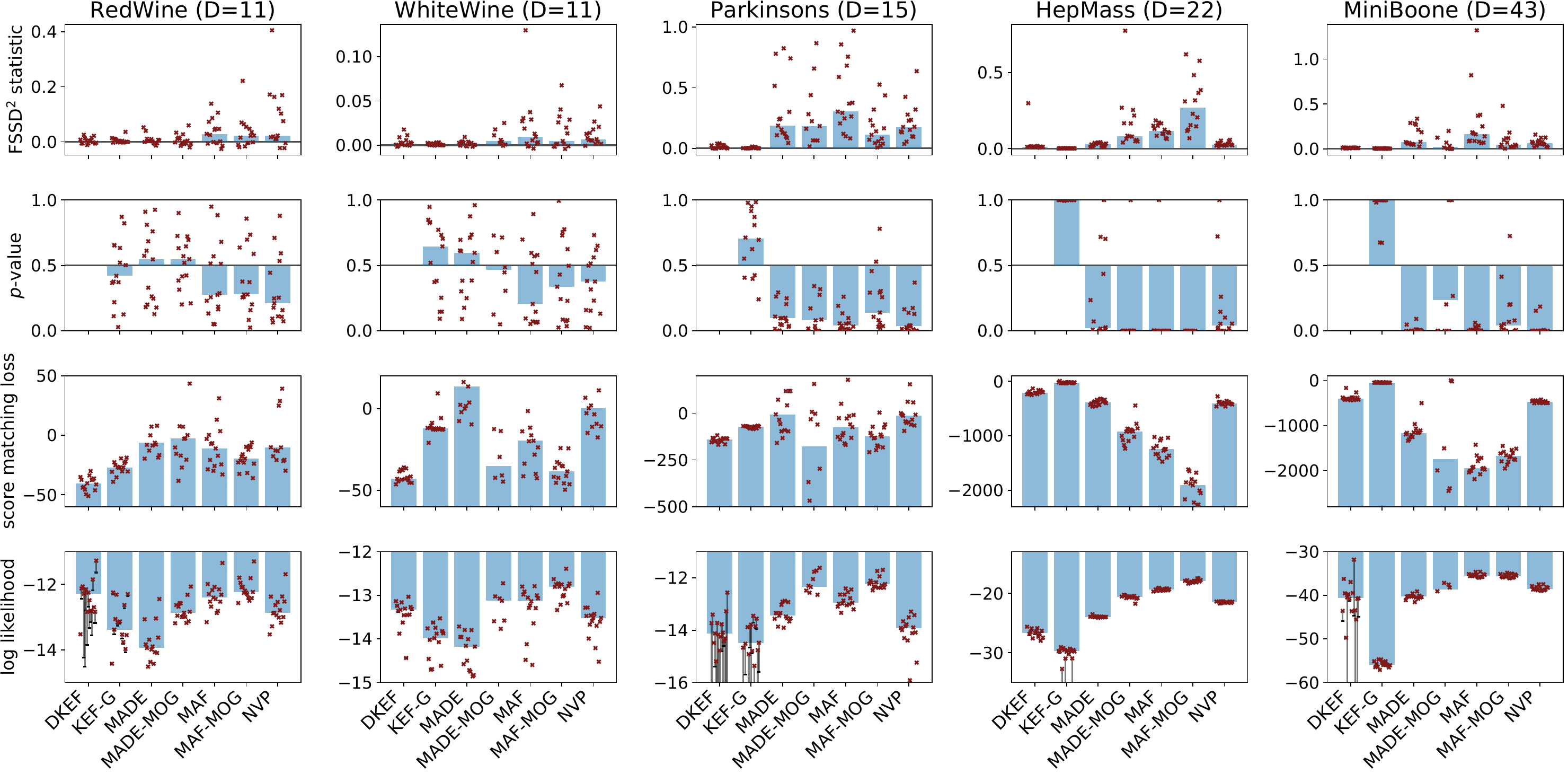}
  \caption{Results on the real datasets;
    bars show medians, points show each of 15 individual runs, excluding invalid values.
    (1st row) The estimate of the squared FSSD, a measure of model goodness of fit based on derivatives of the log density; lower is better.
    (2nd row) The $p$-value of a test that each model is no better than DKEF in terms of the FSSD; values near 0 indicate that DKEF fits the data significantly better than the other model.
    (3nd row) Value of the loss \eqref{eq:score-estimator}; lower is better.
    (4th row) Log-likelihoods; higher is better.
    DKEF estimates are based on $10^{10}$ samples for $\hat Z_\vtheta$, 
    with vertical lines showing the upper bound on the bias from \cref{thm:bias-est}
    (which is often too small to be easily visible).
  }
  \label{fig:real_data}
\end{figure*}

We visualize the fit of various methods by showing the log density function in \cref{fig:toys}.
For each model fit on each distribution, we report the normalized log likelihood (left) and Fisher divergence (right).
In general, the kernel score matching methods find cleaner boundaries of the distributions, and our main model KDEF-G-15 produces the lowest Fisher divergence on many of the synthetic datasets while maintaining high likelihoods.

Among the kernel exponential families,
DKEF-G outperformed 
previous versions where ordinary Gaussian kernels are used for either joint (KEF-G)
or autoregressive (KCEF) modeling.
DKEF-G-50 does not substantially improve over DKEF-G-15; we omit it for space.
We can gain additional insights into the model by looking at the shape of the learned kernel,
shown by the colored lines;
the kernels do indeed adapt to the local geometry. 

DKEF-L-50 and DKEF-L-50-1 show good performance when the target density has simple geometries,
but had trouble in more complex cases,
even with much larger networks than used by DKEF-G-15.
It seems that a Gaussian kernel with inducing points provides much stronger 
representational features than a using linear kernel and/or a single inducing point.
A large enough network $\bphi_\vw$ would likely be able to perform the task well,
but, using currently available software, the second derivatives in the score matching loss
limit our ability to use very large networks.
A similar phenomenon was observed by \citet{demystifying-mmd-gans} in the context of GAN critics,
where combining some analytical RKHS optimization with deep networks
allowed much smaller networks to work well.

As expected, models fit by DKEFs generally have smaller Fisher divergences than likelihood-based methods. 
For Funnel and Banana, the true densities are simple transformations of Gaussians, and the normalizing flow
models perform relatively well. But on Ring, Square, and Cosine,
the shape of the learned distribution by likelihood-based methods exhibits noticeable artifacts, 
especially at the boundaries.
These artifacts, particularly the ``breaks'' in Ring, may be caused by a normalizing flow's need to be invertible and smooth.
The shape learned by DKEF-G-15 is much cleaner.

On multimodal distributions with disjoint components, 
likelihood-based and score matching-based methods show interesting failure modes.
The likelihood-based models often connect separated modes with a ``bridge",
even for MADE-MOG and MAF-MOG which use explicit mixtures.
On the other hand,
DKEF is able to find the shapes of the components, but the weighting between modes is unstable.
As suggested in \cref{sec:mixtures},
we also fit mixtures of all models (except KCEF) on a partition of MoR found by spectral clustering \citep{spectral_clustering}; DKEF-G-15 produced an excellent fit.

Another challenge we observed in our experiments is that the estimator of the objective function, $\hat{J}$, 
tends to be more noisy as the model fit improves.
This happens particularly on datasets where there are ``sharp'' features, 
such as Square (see \cref{fig:toy_traces} in \cref{sec:exp-details-toy}),
where the model's curvature becomes extreme at some points. 
This can cause higher variance in the gradients of the parameters, and  more difficulty in optimization.

\subsection{Results on Benchmark Datasets}

Following recent work on density estimation \citep{rnade, kcef, maf, made}, we trained DKEF and the likelihood-based models on five 
UCI datasets \citep{UCI}; in particular, we used
RedWine, WhiteWine, Parkinson, HepMass, and MiniBoone.
All performances were measured on held-out test sets.
We did not run KCEF due to its computational expense.
\Cref{sec:real-data-details}
gives further details.

\Cref{fig:real_data} shows results.
In gradient matching as measured by the FSSD,
DKEF tends to have the best values.
Test set sizes are too small to yield a confident $p$-value on the Wine datasets,
but the model comparison test confidently favors DKEF on datasets with large-enough test sets.
The FSSD results agree with KSD, which is omitted.
In the score matching loss\footnote{The implementation of \citet{maf} sometimes produced \texttt{NaN} values for the required second derivatives, especially for MADE-MOG. We discarded those runs for these plots.} \eqref{eq:estimate_score},
DKEF is the best on Wine datasets and most runs on Parkinson,
but worse on Hepmass and MiniBoone.
FSSD is a somewhat more ``global'' measure of shape,
and is perhaps more weighted towards the bulk of the distribution rather than the tails.\footnote{%
 With a kernel approaching a Dirac delta,
 the $\mathrm{FSSD}^2$ is similar to
 $\mathrm{KSD}^2 \approx \int \pdata(\x)^2 \, \lVert \nabla\log p(\x) - \nabla\log \ptilde_\vtheta(\x) \rVert^2 \ud\x$;
 compare to $J = \frac12 \int \pdata(\x) \lVert \nabla\log p(\x) - \nabla\log \ptilde_\vtheta(\x) \rVert^2 \ud\x$.}
In likelihoods,
DKEF is comparable to other methods except MADE on Wines but 
worse on the other, larger, datasets. 
Note that we trained DKEF while adding Gaussian 
noise with standard deviation 0.05 to the (whitened) dataset; 
training without noise improves the score matching loss but harms likelihood,
while producing similar results for FSSD.

Results for models with a fixed Gaussian $q_0$ were similar (\cref{fig:real_data_q0}, in \cref{sec:real-data-details}).

\section{Discussion}
Learning deep kernels helps make the kernel exponential family practical for large, complex datasets of moderate dimension.
We can exploit the closed-form fit of the $\valpha$ vector to optimize kernel and even regularization parameters using a ``held-out'' loss,
in a particularly convenient instance of meta-learning.
We are thus able to find smoothness assumptions that fit our particular data,
rather than arbitrarily choosing them a priori.

Computational expense makes score matching with deep kernels difficult to scale to models with large kernel networks,
limiting the dimensionality of possible targets.
Combining with the kernel conditional exponential family might help alleviate that problem by splitting the model up into several separate but marginally complex components.
The kernel exponential family, and score matching in general, also struggles to correctly balance datasets with nearly-disjoint components, but it seems to generally learn density \emph{shapes} better than maximum likelihood-based deep approaches.



\ifdefined\isaccepted
\subsubsection*{Acknowledgments}
This work was supported by the Gatsby Charitable Foundation. We thank Heishiro Kanagawa for helpful discussions.
\fi

\bibliography{ref}
\bibliographystyle{icml2019}

\clearpage
\appendix
\onecolumn

\begin{center}
\Large {\textbf{\mytitle:\\Supplementary material}}
\end{center}

\section{DKEFs can be normalized} \label{sec:normalizable}

\begin{prop} \label{thm:normalizable}
  Consider the kernel $k(\x, \y) = \kappa(\bphi(\x), \bphi(\y))$,
  where $\kappa$ is a kernel such that
  $\kappa(\va, \va) \le L_\kappa \lVert \va \rVert^2 + C_\kappa$
  and $\bphi$ a function such that $\lVert \bphi(\x) \rVert \le L_{\bphi} \lVert \x \rVert + C_{\bphi}$.
  Let $q_0(\x) = Q \, r_0(\mV^{-1} (\x - \vmu))$,
  where $Q > 0$ is any scalar
  and $r_0$ is a product of independent generalized Gaussian densities,
  with each $\beta_d > 1$:
  \[
    r_0(\z) = \prod_{d=1}^D
      \frac{\evbeta_d}{2 \operatorname{\Gamma}\left(\frac{1}{\beta_d}\right)}
      \exp\left( - \lvert z_d \rvert^{\beta_d} \right)
  .\]
  (For example, $\mathcal{N}(\vmu, \mSigma)$ for strictly positive definite $\mSigma$ could be achieved with $\beta_d = 2$ and $\mV$ the Cholesky factorization of $\mSigma$.)
  Then, for any function $f$ in the RKHS $\cH$ corresponding to $k$,
  \[
    \int \exp(f(\x)) \,q_0(\x) \,\ud\x < \infty
  .\]
\end{prop}
\begin{proof}
First, we have that
$
f(\x)
= \langle f, k(\x, \cdot) \rangle_\cH
\le \lVert f \rVert_\cH \sqrt{k(\x, \x)}
,$
and
\[
k(\x, \x)
= \kappa(\bphi(\x), \bphi(\x))
\le L_\kappa \lVert \bphi(\x) \rVert^2 + C_\kappa
\le L_\kappa (L_{\bphi} \lVert \x \rVert^2 + C_{\bphi}) + C_\kappa
.\]
Combining these two yields
\[
  f(\x)
  \le \lVert f \rVert_\cH \sqrt{L_\kappa L_{\bphi} \lVert\x\rVert^2 + L_\kappa C_{\bphi} + C_\kappa}
  \le \lVert f \rVert_\cH \sqrt{L_\kappa L_{\bphi}} \lVert\x\rVert 
    + \lVert f \rVert_\cH \sqrt{L_\kappa C_{\bphi} + C_\kappa}
  \le C_0 + C_1 \lVert\vx\rVert
,\]
defining $C_1 := \lVert f \rVert_\cH \sqrt{L_\kappa L_{\bphi}}$,
$C_0 := \lVert f \rVert_\cH \sqrt{L_\kappa C_{\bphi} + C_\kappa}$.

Let $\z = \mV^{-1}(\x - \vmu)$,
and let $C_r$ be the normalizing constant of $r_0$,
$C_q := \prod_{d=1}^D \frac{\evbeta_d}{2 \alpha_d \operatorname{\Gamma}\left(\frac{1}{\beta_d}\right)}$.
Then
\begin{align*}
       \int \exp(f(\x)) \,q_0(\x) \,\ud\x
  &\le \int \exp\left( C_0 + C_1 \lVert\x\rVert \right) q_0(\x) \ud\x
\\&  = Q \exp(C_0) \E_{\z \sim r_0}\left[ \exp\left( C_1 \lVert \mV \z + \vmu \rVert \right) \right]
\\&\le Q \exp(C_0 + C_1 \lVert\vmu\rVert) \E_{\z \sim r_0}\left[ \exp\left( C_1 \lVert \mV \rVert \lVert \z \rVert \right) \right]
\\&\le Q \exp(C_0 + C_1 \lVert\vmu\rVert)
        \E_{\z \sim r_0}\left[ \exp\left( C_1 \lVert \mV \rVert \; \sum_{d=1}^D \lvert z_d \rvert \right) \right]
\\&  = Q \exp(C_0 + C_1 \lVert\vmu\rVert)
       \prod_{d=1}^D
        \E_{\z \sim r_0}\left[ \exp\left( C_1 \lVert \mV \rVert \lvert z_d \rvert \right) \right]
.\end{align*}
We can now show that each of these expectations is finite:
letting $C = C_1 \lVert \mV \rVert$,
\begin{align*}
        \E_{\z \sim r_0}\left[ \exp\left( C \lvert z_d \rvert \right) \right]
  &  = \int_{-\infty}^\infty \exp\left( C \lvert z \rvert \right) \cdot \frac{\beta}{2 \Gamma(1 / \beta)} \exp\left( - \lvert z \rvert^\beta \right) \,\ud z
\\&  = 2 \frac{\beta}{2 \Gamma(1 / \beta)} \int_0^\infty \exp\left( C z - z^\beta \right) \,\ud z
\\&  = 2 \frac{\beta}{2 \Gamma(1 / \beta)} \left(
            \int_0^s \exp\left( C z - z^\beta \right) \,\ud z
          + \int_s^\infty \exp\left( C z - z^\beta \right) \,\ud z
       \right)
\end{align*}
for any $s \in (0, \infty)$.
The first integral is clearly finite.
Picking $s = \left(2 \lvert C \rvert \right)^{\frac{1}{\beta - 1}}$,
so that $\lvert C z \rvert < \frac12 z^\beta$ for $z > s$,
gives that
\begin{align*}
      \int_s^\infty \exp\left( C z - z^\beta \right) \,\ud z
  \le \int_s^\infty \exp\left( - \tfrac12 z^\beta \right) \,\ud z
    < \frac1\beta 2^{\frac1\beta} \Gamma\left(\frac1\beta \right)
    < \infty
,\end{align*}
so that $\int \exp(f(\x)) q_0(\x) \ud\x < \infty$ as desired.
\end{proof}

The condition on $\bphi$ holds for any $\bphi$ given by a deep network with Lipschitz activation functions,
such as the softplus function we use in this work.
The condition on $\kappa$ also holds for a linear kernel (where $L_\kappa = 1$, $C_\kappa = 0$),
any translation-invariant kernel ($L_\kappa = 0$, $C_\kappa = \kappa(0, 0)$),
or mixtures thereof.
If $\kappa$ is bounded, the integral is finite for any function $\bphi$. 

The given proof would not hold for a quadratic $\kappa$, which has been used previously in the literature;
indeed, it is clearly possible for such an $f$ to be unnormalizable.

\section{Finding the optimal \texorpdfstring{$\valpha$}{alpha}} \label{sec:alpha-proof}

We will show a slightly more general result than we need,
also allowing for an $\lVert f \rVert_\cH^2$ penalty.
This result is related to Lemma 4 of \citet{SutStrArb18},
but is more elementary and specialized to our particular needs
while also allowing for more types of regularizers.

\begin{prop} \label{thm:alpha}
  Consider the loss
  \[
    \hat J(f_{\valpha,\z}^k, \vlambda, \D)
    =
    \hat J(p_{\valpha,\z}^k, \D)
    + \frac12 \left[
      \lambda_\alpha \lVert \valpha \rVert^2
      + \lambda_\cH \lVert f_{\valpha,\z}^k \rVert_\cH^2
      + \lambda_C \frac1N \sum_{n=1}^N \sum_{d=1}^D \left[
        \partial_d^2 \log \ptilde_{\valpha,\z}^k(\xn)
      \right]^2
    \right]
  \]
  where
  \[
    \hat J(p_{\valpha,\z}^k, \D)
    =
    \frac1N \sum_{n=1}^N \sum_{d=1}^D \left[
      \partial_d^2 \log \ptilde_{\valpha,\z}^k(\x_n)
      + \frac12 \left( \partial_d \log \ptilde_{\valpha,\z}^k(\x_n) \right)^2
    \right]
  .\]
  For fixed $k$, $\z$, and $\vlambda$,
  as long as $\lambda_\valpha > 0$ then
  the optimal $\valpha$ is
  \begin{align*}
    \valpha(\vlambda, k, \vz, \D)
    &= \argmin_{\valpha} \hat J(f_{\valpha,\vz}^k, \vlambda, \D)
    = -\left(
      \mG + \lambda_\alpha \mI
      + \lambda_\cH \mK 
      + \lambda_C \mU
    \right)^{-1} \vb
  \\
      \emG_{m,m'}
      &= \frac1N \sum_{n=1}^N \sum_{d=1}^D 
            \partial_d k(\xn, \zm) \, \partial_d k(\xn, \z_{m'})
  \\
      \emU_{m,m'}
      &= \frac1N \sum_{n=1}^N \sum_{d=1}^D
            \partial_d^2 k(\xn, \zm) \, \partial_d^2 k(\xn, \z_{m'})
  \\
      \emK_{m,m'} &= k(\z_m, \z_{m'})
  \\
      \evb_{m}
      &= \frac1N \sum_{n=1}^N \sum_{d=1}^D \partial^2_d k(\xn, \zm)
              + \partial_d \log q_0(\xn) \, \partial_d k(\xn, \zm)
              + \lambda_C \partial_d^2 \log q_0(\xn) \, \partial_d^2 k(\xn, \zm)
  .\end{align*}
\end{prop}
\begin{proof}
  We will show that the loss is quadratic in $\valpha$.
  Note that
  \begin{align*}
       \frac1N \sum_{n=1}^N \sum_{d=1}^D \partial_d^2 \log \ptilde_{\valpha,\z}^k(\xn)
    &= \frac1N \sum_{n=1}^N \sum_{d=1}^D \left[
         \sum_{m=1}^M \alpha_m \partial_d^2 k(\xn, \zm)
       + \partial_d^2 \log q_0(\xn)
       \right]
  \\&= \valpha\tp \left[
         \frac1N \sum_{n=1}^N \sum_{d=1}^D \partial_d^2 k(\xn, \zm)
       \right]_m
       + \text{const}
  \\   \frac1N \sum_{n=1}^N \sum_{d=1}^D \frac12 \left( \partial_d \log \ptilde_{\valpha,\z}^k(\xn) \right)^2
    &= \frac1N \sum_{n=1}^N \sum_{d=1}^D \frac12 \left(
         \sum_{m, m'=1}^M \alpha_m \alpha_{m'} \partial_d k(\xn, \zm) \partial_d k(\xn, \z_{m'})
\right.\\&\qquad\left.
       + 2 \sum_{m=1}^M \alpha_m \partial_d \log q_0(\xn) \partial_d k(\xn, \zm) 
       + \left(\partial_d \log q_0(\xn)\right)^2
       \right)
  \\&= \frac12 \valpha\tp \mG \valpha + \valpha\tp \left[ \frac1N \sum_{n=1}^N \sum_{d=1}^D \partial_d \log q_0(\xn) \partial_d k(\xn, \zm) \right] + \text{const}
  .\end{align*}
  The $\lambda_C$ term is of the same form, but with second derivatives:
  \begin{align*}
       \frac{1}{2 N} \sum_{n=1}^N \sum_{d=1}^D \left( \partial_d^2 \log \ptilde_{\valpha,\z}^k(\xn) \right)^2
    &= \frac12 \valpha\tp \mU \valpha + \valpha\tp \left[
         \frac1N \sum_{n=1}^N \sum_{d=1}^D \partial_d^2 \log q_0(\xn) \partial_d^2 k(\xn, \zm)
       \right]
       + \text{const}
  .\end{align*}
  We also have as usual
  \begin{align*}
       \frac12 \lVert f_{\valpha,\z}^k \rVert_\cH^2
    &= \frac12 \sum_{m=1}^M \sum_{m'=1}^M \alpha_m \langle k(\zm, \cdot), k(\z_{m'}, \cdot) \rangle_\cH \, \alpha_{m'}
     = \frac12 \valpha\tp \mK \valpha
  .\end{align*}
  Thus the overall optimization problem is
  \begin{align*}
       \valpha(\vlambda, k, \vz, \D)
    &= \argmin_{\valpha} \hat J(f_{\valpha,\vz}^k, \vlambda, \D)
  \\&= \argmin_{\valpha} \frac12 \valpha\tp \left( \mG + \lambda_\valpha \mI + \lambda_\cH \mK + \lambda_C \mU \right) \valpha + \valpha\tp \vb
  .\end{align*}
  Because $\lambda_\alpha > 0$ and $\mG$, $\mK$, $\mU$ are all positive semidefinite,
  the matrix in parentheses is strictly positive definite,
  and the claimed result follows directly from standard vector calculus.
\end{proof}

\section{Behavior on mixtures} \label{sec:mixture-details}

\begin{prop} \label{thm:mixtures}
  Let $\D = \bigcup_{i=1}^I \D_i$,
  where $\D_i \subset \mathcal X_i$, $\lvert \D_i \rvert = \pi_i N$, $\sum_{i=1}^I \pi_i = 1$.
  Also suppose that the inducing points are partitioned as
  $\mZ = \left[ \mZ_1; \dots; \mZ_I \right]$,
  with $\mZ_i \subset \mathcal X_i$.
  Further let the kernel $k$ be such that
  $k(\x_1, \x_2) = 0$ when $\x_1 \in \mathcal X_i$, $\x_2 \in \mathcal X_j$ for $i \ne j$,
  with its first and second derivatives also zero.
  Then the kernel exponential family solution of \cref{thm:alpha} is
  \[
    \valpha(\vlambda, k, \z, \D) = \begin{bmatrix}
      \valpha\left(
        \left( \frac{\lambda_\alpha}{\pi_1}, \frac{\lambda_\cH}{\pi_1}, \lambda_C \right),
        k, \mZ_1, \D_1
      \right) \\
      \vdots \\
      \valpha\left(
        \left( \frac{\lambda_\alpha}{\pi_I}, \frac{\lambda_\cH}{\pi_I}, \lambda_C \right),
        k, \mZ_I, \D_I
      \right) \\
    \end{bmatrix}.
  \]
\end{prop}
\begin{proof}
  Let $\mG_i$, $\vb_i$ be the $\mG$, $\vb$ of \cref{thm:alpha} when using only $\mZ_i$ and $\D_i$.
  Then, because the kernel values and derivatives are zero across components,
  if $m$ and $m'$ are from separate components then
  \[
    \emG_{m,m'} = \frac1N \sum_{n=1}^N \sum_{d=1}^D \partial_d k(\xn, \zm) \partial_d k(\xn, \z_{m'}) = 0
  ,\]
  as at least one of the kernel derivatives will be zero for each term of the sum.
  When $m$ and $m'$ are from the same component,
  the total will be the same except that $N$ is bigger, giving
  \[
    \mG = \begin{bmatrix}
      \pi_1 \mG_1 & {\bm 0}     & \cdots & {\bm 0} \\
      {\bm 0}     & \pi_2 \mG_2 & \cdots & {\bm 0} \\
      \vdots      & \vdots      & \ddots & \vdots \\
      {\bm 0}     & {\bm 0}     & \cdots & \pi_I \mG_I \\
    \end{bmatrix}
  .\]
  $\mU$ is of the same form and factorizes in the same way.
  $\mK$ does not scale:
  \[
    \mK = \begin{bmatrix}
      \mK_1 & \dots & {\bm 0} \\
      \vdots & \ddots & \vdots \\
      {\bm 0} & \cdots & \mK_I
    \end{bmatrix}
  .\]
  Recall that $\vb$ is given as
  \[
    \evb_{m}
     = \frac1N \sum_{n=1}^N \sum_{d=1}^D \partial^2_d k(\xn, \zm)
            + \partial_d \log q_0(\xn) \, \partial_d k(\xn, \zm)
            + \lambda_C \partial_d^2 \log q_0(\xn) \, \partial_d^2 k(\xn, \zm)
  .\]
  Each term in the sum for which $\xn$ is in a different component than $\zm$ will be zero,
  giving $\vb = \left( \pi_1 \vb_1, \cdots, \pi_I \vb_I \right)$.
  Thus $\valpha(\vlambda, k, \vz, \D)$ becomes
  \begin{align*}
       \valpha
    &= - \left( \mG + \lambda_\alpha \mI + \lambda_\cH \mK + \lambda_C \mU \right)^{-1} \vb
  \\&= - \begin{bmatrix}
         \pi_1 \mG_1 + \lambda_\alpha \mI + \lambda_\cH \mK_1 + \lambda_C \pi_1 \mU_1 & \cdots & {\bm 0} \\
         \vdots & \ddots & \vdots \\
         {\bm 0} & \cdots & \pi_I \mG_I + \lambda_\alpha \mI + \lambda_\cH \mK_I + \lambda_C \pi_I \mU_I \\
       \end{bmatrix}^{-1}
       \begin{bmatrix}
         \pi_1 \vb_1 \\ \vdots \\ \pi_I \vb_I
       \end{bmatrix}
  \\&= \begin{bmatrix}
         -(\mG_1 + \frac{\lambda_\alpha}{\pi_1} \mI + \frac{\lambda_\cH}{\pi_1} \mK_1 + \lambda_C \mU_1)^{-1} \vb_1 \\
         \vdots \\
         -(\mG_I + \frac{\lambda_\alpha}{\pi_I} \mI + \frac{\lambda_\cH}{\pi_I} \mK_I + \lambda_C \mU_I)^{-1} \vb_2
       \end{bmatrix}
  .\qedhere\end{align*}
\end{proof}

Thus the fits for the components are essentially added together,
except that each component uses a different $\lambda_\alpha$ and $\lambda_\cH$;
smaller components are regularized more.
$\lambda_C$, interestingly, is unscaled.

It is difficult in general to tell how two components will be weighted relative to one another;
the problem is essentially equivalent to computing the overall normalizing constant of a fit.
However, we can gain some insight by analyzing a greatly simplified case, in \cref{sec:fit-gaussians}.

\subsection{Small Gaussian components with a large Gaussian kernel} \label{sec:fit-gaussians}
\begingroup 
Consider, for the sake of our study of mixture fits,
one of the simplest possible situations for a kernel exponential family:
$\pdata = \N(\vzero, \mI)$,
with a kernel $k(\x, \y) = \exp\left( - \frac{1}{\sigma^2} \lVert \x - \y \rVert^2 \right)$
for $\sigma \gg \sqrt D$,
so that $k(\x, \y) \A 1$ for all $\x, \y$ sampled from $\pdata$.
Let $q_0$ be approximately uniform, $q_0 = \N(\vzero, q \mI)$ for $q \gg \sigma^2$,
so that $\nabla \log q_0(\x) = \frac{-1}{q^2} \x \A 0$.
Also assume that $N \to \infty$, but $M$ is fixed.
Assume that $\lambda_{\cH} = \lambda_C = 0$, and refer to $\lambda_\alpha$ as simply $\lambda$.
Then we have
\begin{align*}
      \emG_{m,m'}
  & = \frac1N \sum_{n=1}^N \sum_{d=1}^D \partial_d k(\x_n, \z_m) \partial_d k(\x_n, \z_{m'})
\\& = \frac1N \sum_{n=1}^N \sum_{d=1}^D
      \left( \frac{\evz_{m,d} - \ervx_d}{\sigma^2} k(\x_n, \z_m) \right)
      \left( \frac{\evz_{m',d} - \ervx_d}{\sigma^2} k(\x_n, \z_{m'}) \right)
\\&\A \sigma^{-4} \E_{\rvx \sim \pdata} \left[
      (\vz_m - \rvx)\tp (\vz_{m'} - \rvx)
      \right]
\\& = \sigma^{-4} \left( \vz_m\tp \vz_{m'} + D \right)
\\    \mG
  &\A \frac{1}{\sigma^4} \mZ \mZ\tp + \frac{D}{\sigma^4} \mI
\\    \evb_m
  & = \frac1N \sum_{n=1}^N \sum_{d=1}^D \partial^2_d k(\xn, \zm) + \partial_d \log q_0(\xn) \, \partial_d k(\xn, \zm)
\\&\A \E_{\rvx \sim \pdata} \left[ \sum_{d=1}^D \partial^2_d k(\xn, \zm) \right]
\\& = \E_{\rvx \sim \pdata} \left[ \sum_{d=1}^D \left(
        \frac{(\ervx_d - \evz_{m,d})^2}{\sigma^4} - \frac{1}{\sigma^2}
      \right)
      k(\rvx, \zm)
      \right]
\\&\A \E_{\rvx \sim \pdata} \left[
        \frac{\lVert \rvx_d - \vz_{m} \rVert^2}{\sigma^4} - \frac{D}{\sigma^2}
      \right]
\\& = \frac{\lVert \vz_{m} \rVert^2 + D}{\sigma^4} - \frac{D}{\sigma^2}
\\    \vb
  &\A \frac{1}{\sigma^4} \diag(\mZ \mZ\tp)
    - \frac{D (\sigma^2 - 1)}{\sigma^4} \vone
.\end{align*}

Because $k(\z, \x) \A k(\z', \x)$ for any $\z, \z'$ near the data in this setup,
it's sufficient to just consider a single $\z = \vzero$.
In that case,
\begin{align*}
      \valpha
  & = -(\mG + \lambda \mI)^{-1} \vb 
\\&\A -\left( \frac{1}{\sigma^4} \mZ \mZ\tp + \left( \frac{D}{\sigma^4} + \lambda \right) \mI \right)^{-1}
      \left( \frac{1}{\sigma^4} \diag(\mZ \mZ\tp) - \frac{D (\sigma^2 - 1)}{\sigma^4} \vone \right)
\\& = -\left( \left( \frac{D}{\sigma^4} + \lambda \right) \mI \right)^{-1}
      \left( - \frac{D (\sigma^2 - 1)}{\sigma^4} \vone \right)
\\& = \frac{1}{D \sigma^{-4} + \lambda} \frac{D (\sigma^2 - 1)}{\sigma^4} \vone
\\& = \frac{\sigma^2 - 1}{1 + \lambda \sigma^{4} / D} \vone
\end{align*}
and so
\begin{align*}
      f_\valpha(\vzero)
  &\A \frac{\sigma^2 - 1}{1 + \lambda \sigma^{4} / D}
.\end{align*}

Thus, if we attempt to fit the mixture
$\pi \N(\vzero, \mI) + (1-\pi) \N(\vr, \mI)$
with $q^2 \gg \lVert \vr \rVert^2 \gg \sigma^2 \gg D$,
we are approximately in the regime of \cref{thm:mixtures}
and so the ratio between the two components in the fit is
\begin{align*}
      \exp\left( f(\vzero) - f(\vr) \right)
  &\A \exp\left(
        \frac{\sigma^2 - 1}{1 + \frac{\lambda \sigma^{4}}{\pi D}}
      - \frac{\sigma^2 - 1}{1 + \frac{\lambda \sigma^{4}}{(1-\pi) D}}
      \right)
\\& = \exp\left( \lambda \sigma^4 (\sigma^2 - 1) \left(
        \frac{%
          \frac{1}{(1-\pi) D} - \frac{1}{\pi D}
        }{%
          1 + \frac{\lambda \sigma^4}{\pi D} + \frac{\lambda \sigma^4}{(1-\pi) D}
          + \frac{\lambda^2 \sigma^8}{\pi (1 - \pi) D^2}
        }
      \right) \right)
\\& = \exp\left( \frac12 \lambda \sigma^4 (\sigma^2 - 1) \left(
        \frac{%
          \pi - \frac12
        }{%
          D \pi (1-\pi)
        + \lambda \sigma^4
        + \frac{\lambda^2 \sigma^8}{D}
        }
      \right) \right)
.\end{align*}
If $\pi = \frac12$, the density ratio is correctly 1.
If we further assume that $\lambda \gg D / \sigma^4$,
so that the denominator is dominated by the last term,
then the ratio becomes approximately
\[
  \exp(f(\vzero) - f(\vr))
  \A
  \exp\left( \frac{D}{2 \sigma^2 \lambda} \left( \pi - \frac12 \right) \right)
.\]
Thus, depending on the size of $D / (2 \sigma^2 \lambda) \ll \sigma^2 / 2$,
the ratio will usually either remain too close to $\frac12$
or become too extreme as $\pi$ changes;
only in a very narrow parameter range is it approximately correct.
\endgroup


\section{Upper bound on normalizer bias} \label{sec:bias-proofs}

Recall the importance sampling setup of \cref{sec:model-eval}:
\[
  \hat Z_\vtheta
  = \frac{1}{U} \sum_{u=1}^U \rr_u
  \quad\text{ where }
  \y_u \sim q_0,
  \rr_u := \frac{\ptilde_\vtheta(\y_u)}{q_0(\y_u)}
  \quad\text{so}\quad
  \E \hat Z_\vtheta
  = \int \frac{\ptilde_\vtheta(\y_u)}{q_0(\y_u)} q_0(\y_u)
  = Z_\vtheta
.\]

\biasestthm*
\begin{proof}
  Inspired by the technique of \citet{sharpening-jensen},
  we will decompose the bias as follows.
  (We will suppress the subscript $\vtheta$ for brevity.)

  First note that the following form of a Taylor expansion holds identically:
  \begin{align*}
        \varphi(x)
    &\phantom{:}=
        \varphi(Z) + \varphi'(Z) (x - Z) + h(x, Z) (x - Z)^2
  \\    h(x, Z)
    &:= \frac{\varphi(x) - \varphi(Z) - \varphi'(Z) (x - Z)}{(x - Z)^2}
  .\end{align*}
  We can thus write the bias as the following,
  where $\varphi(x) = -\log(x)$,
  $P$ is the distribution of $\hat Z$,
  and $t \ge a$:
  \begin{align*}
         \E[\varphi(\hat Z)] - \varphi(\E \hat Z)
    &  = \int_a^\infty \left( \varphi(x) - \varphi(Z) \right) \ud P(x)
  \\&  = \int_a^\infty \left( \varphi'(Z) (x - Z) + h(x, Z) (x - Z)^2 \right) \ud P(x)
  \\&  = \varphi'(Z) \left( \int_a^\infty x \,\ud P(x) - Z \right)
       + \int_a^\infty h(x, Z) (x - Z)^2 \ud P(x)
  \\&  = \int_a^t h(x, Z) (x - Z)^2 \ud P(x)
       + \int_t^\infty h(x, Z) (x - Z)^2 \ud P(x)
  \\&\le \left[ \sup_{a \le x \le t} h(x, Z) (x - Z)^2 \right] \int_a^t \ud P(x)
       + \left[ \sup_{x \ge t} h(x, Z) \right] \int_t^\infty (x - Z)^2 \ud P(x)
  \\&\le \left[ \sup_{a \le x \le t} h(x, Z) (x - Z)^2 \right] \Pr(\hat Z \le t)
       + \left[ \sup_{x \ge t} h(x, Z) \right] \Var[\hat Z]
  .\end{align*}
  Now,
  \[
    h(x, Z) (x - Z)^2
    = \log \frac{Z}{x} + \frac{x}{Z} - 1
  \]
  is convex in $x$ and thus its supremum is
  $
    \max\left(
      \log \frac{Z}{a} + \frac{a}{Z} - 1,
      \log \frac{Z}{t} + \frac{t}{Z} - 1
    \right)
  $,
  with the term at $a$ being necessarily larger as long as $t < Z$.

  Picking $t = (s + a) / 2$
  gives the desired bound on
  $\Pr(\hat Z \le t)$
  via \cref{thm:mean-bound}.

  Lemma 1 of \citet{sharpening-jensen} shows that since $\varphi'(x) = - 1 / x$ is concave, 
  $h(x, Z)$ is decreasing in $x$.
  Thus $\sup_{x \ge t} h(x, Z) = h(t, Z)$,
  giving the claim.
\end{proof}



\begin{lemma} \label{thm:mean-bound}
  Let $a$ and $s$ be such that $\Pr(\rr_u \ge a) = 1$ and $\Pr(\rr_u \le s) \le \rho < \frac12$,
  with $a < s$.
  Then
  \[
    \Pr\left( \frac{1}{U} \sum_{i=1}^U \rr_u \le \frac{s + a}{2} \right)
    \le (4 \rho (1 - \rho))^{\frac{U}{2}}
  .\]
\end{lemma}
\begin{proof}
  Let $K$ denote the number of samples of $\rr_u$ which are smaller than $s$,
  so that $U - K$ samples are at least $s$.
  Then we have
  \begin{align*}
         \Pr\left( \frac1U \sum_{u=1}^U \rr_u \le \frac{s + a}{2} \right)
    &\le \Pr\left( \frac{K}{U} a + \frac{U - K}{U} s \le \frac{s + a}{2} \right)
  \\&  = \Pr\left( K (a - s) \le U \frac{a - s}{2} \right)
  \\&  = \Pr\left( K \ge \frac{U}{2} \right)
  .\end{align*}
  $K$ is distributed binomially with probability of success at most $\rho < \frac12$,
  so applying Theorem 1 of \citet{binomial-large-dev} yields
  \begin{align*}
         \Pr\left( \frac1U \sum_{i=1}^U \rr_u \le \frac{s + a}{2} \right)
    &\le \exp\left( - U \left[
           \frac12 \log\frac{1}{2 \rho}
         + \frac12 \log\frac{1}{2 (1- \rho)}
         \right] \right)
  \\&  = \exp\left( \frac{U}{2} \log\left( 4 \rho (1 - \rho) \right) \right)
  \\&  = \left( 4 \rho (1 - \rho) \right)^{\frac{U}{2}}
  .\qedhere \end{align*}
\end{proof}

\begin{prop} \label{thm:bias-est-bias}
  The function $\chi_t(x) := \left( \log \frac{x}{t} + \frac{t}{x} - 1 \right) / (x - t)^2$
  is strictly convex for all $x > 0$.
  Thus we have that $\E \chi_t(\hat Z_\vtheta) \ge \chi_t(\E \hat Z_\vtheta) = \chi_t(Z_\vtheta)$,
  with equality only if $\Pr(\hat Z_\vtheta = Z_\vtheta) = 1$.
\end{prop}
\begin{proof}
  We can compute that
  \begin{align*}
       \chi_t''(x)
    &= \frac{2 \frac{t^3}{x^3} - 9 \frac{t^2}{x^2} + 18 \frac{t}{x} - 11 - 6 \log \frac{t}{x}}{(x - t)^4}
  .\end{align*}
  Let $r := t / x$,
  so $x \in [t, \infty)$ corresponds to $r \in (0, 1]$,
  and $x \in (0, t]$ corresponds to $r \in [1, \infty)$.
  Then
  \begin{align*}
       \chi_t''\left( \frac{t}{r} \right)
    &= \frac{2 r^3 - 9 r^2 + 18 r - 11 - 6 \log r}{t^4 \left( \frac{1}{r} - 1 \right)^4}
  .\end{align*}
  We can evaluate
  $\lim_{r \to 1} \chi''(t/r) = \frac{3}{2 t^4} > 0$.
  For $r \ne 1$,
  $\chi_t'' > 0$ if and only if $f(r) > 0$, where
  \[ f(r) := 2 r^3 - 9 r^2 + 18 r - 11 - 6 \log r . \]
  Clearly $\lim_{r \to 0} f(r) = \infty$ and $f(1) = 0$.
  But notice that
  \[
    f'(r)
    = 6 r^2 - 18 r + 18 - \frac{6}{r}
    = \frac{6 (r - 1)^3}{r}
  ,\]
  so that $f(r)$ is strictly decreasing on $(0, 1)$, and strictly increasing on $(1, \infty)$.
  Thus $f(r) > 0$ for all $r \in (0, 1) \cup (1, \infty)$,
  and $\chi_t''(x) > 0$ for all $x > 0$.
  The claim follows by Jensen's inequality.
\end{proof}

\subsection{Estimator of bias bound} \label{sec:bias-estimator}

For a kernel such as \eqref{eq:k_theta} bounded in $[0, 1]$,
$a := \exp\left( \sum_{m=1}^M \min(\alpha_m, 0) \right)$
is a uniform lower bound on $\rr_u$.

For large $U$, essentially any $\rho < \frac12$ will make the second term practically zero,
so we select $s$ as slightly less than the 40th percentile of an initial sample of $\rr_u$,
and confirm a high-probability ($0.999$) Hoeffding upper bound $\rho$ on $\Pr(\rr_u \le s)$ with another sample.
(We use $s$ as $\exp(-0.001) \approx 0.999$ times the estimate of the 40th percentile, to avoid ties.)
We use $10^7$ samples for each of these.

We estimate $\Var[\rr_u]$ on a separate sample with the usual unbiased estimator,
using $10^9$ samples for most cases but $10^{10}$ for MiniBoone.

To finally estimate the bound,
we estimate $Z_\vtheta$ on yet another independent sample,
again usually of size $10^9$ but $10^{10}$ for MiniBoone.

Crucially, the function $\psi(t, x) / (x - t)^2$ is convex (\cref{thm:bias-est-bias});
because the variance is unbiased,
our estimate of the bias bound is itself biased upwards.
As \cref{thm:bias-est}'s bound is also not tight,
our estimate thus likely overstates the actual amount of bias.

\section{Additional experimental details}\label{sec:exp-details}

\subsection{Synthetic datasets}\label{sec:exp-details-toy}
For each synthetic distribution, we sample $10\,000$ random points from the distribution,
$1\,000$ of which are used for testing;
of the rest, 90\% ($8\,100$) are used for training, and 10\% (900) are used for validation. 
Training was early stopped when validation cost does not improve for 200 minibatches.
The current implementation of KCEF does not include a Nystr\"om approximation,
and trains via full-batch L-BFGS-B,
so we down-sampled the training data to 1000 points. 
We used the Adam optimizer \citep{adam} for all other models.
For MADE, RealNVP, and MAF, we used minibatches of size 200 and the learning rate was $10^{-3}$
For KEF-G and DKEF, we used 200 inducing points, 
used $|\Dt|=|\Dv|=100$, and learning rate $10^{-3}$. The same parameters are used for each component
for mixture models trained on MoR.

To show that learning is stable, we ran the experiments on 5 random draws of training, 
validation and test sets from the synthetic distributions, 
trained KDEF initialized using 5 random seeds and calculated validation score at each iteration until convergence in
the first phase of training (before optimizing for $\lambda$'s).
The traces are shown in \cref{fig:toy_traces}. 

The same data for benchmark datasets are shown in \cref{fig:real_traces}. There is no overfitting except for the small Redwine dataset. 
Runs on Hepmass and Miniboone do not seem to fully converge, despite having met the early stopping criterion.

\begin{figure}[t]
    \centering
    \includegraphics[width=\textwidth]{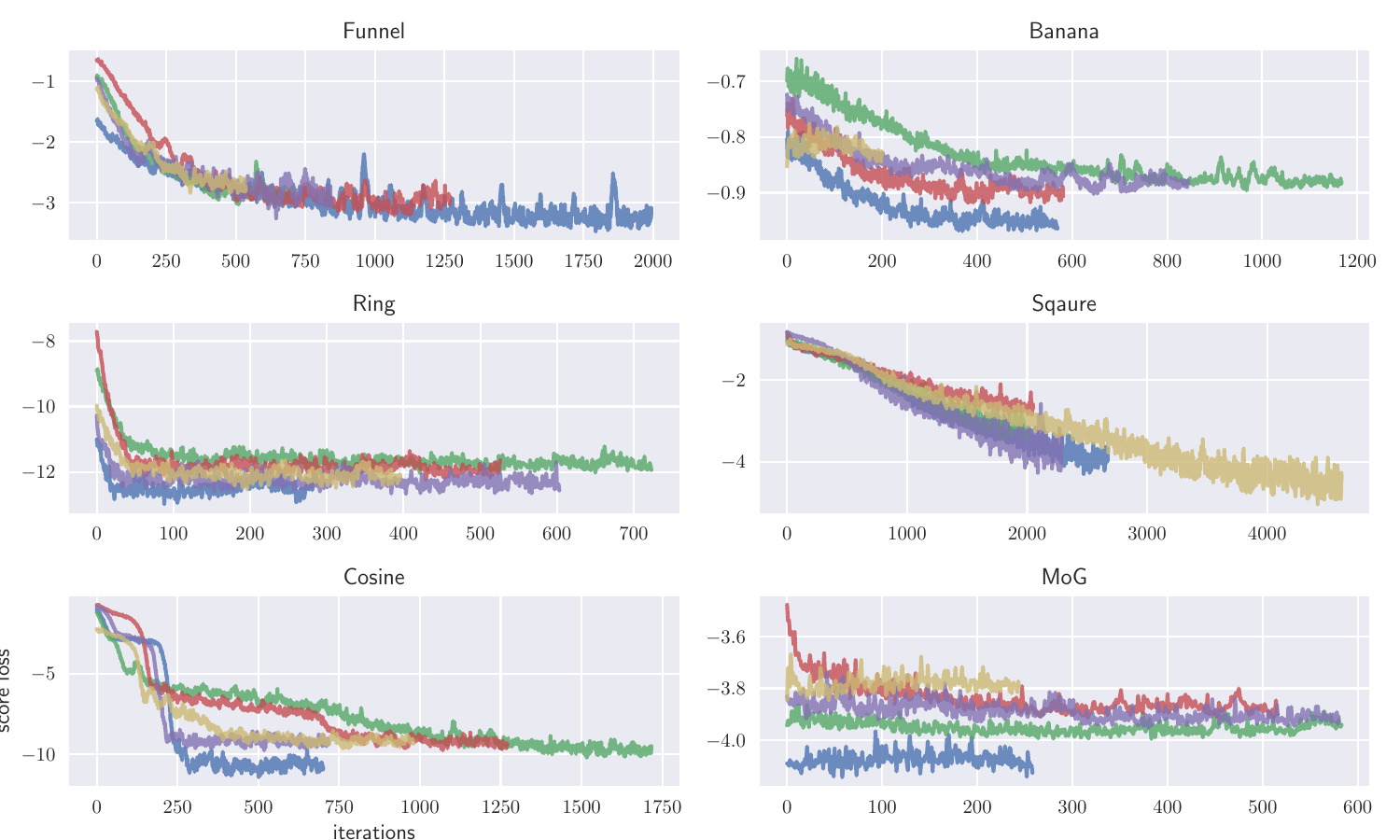}
    \caption{Validation score loss on 6 synthetic datasets for 5 runs.}
    \label{fig:toy_traces}
\end{figure}

\begin{figure}[t]
    \centering
    \includegraphics[width=\textwidth]{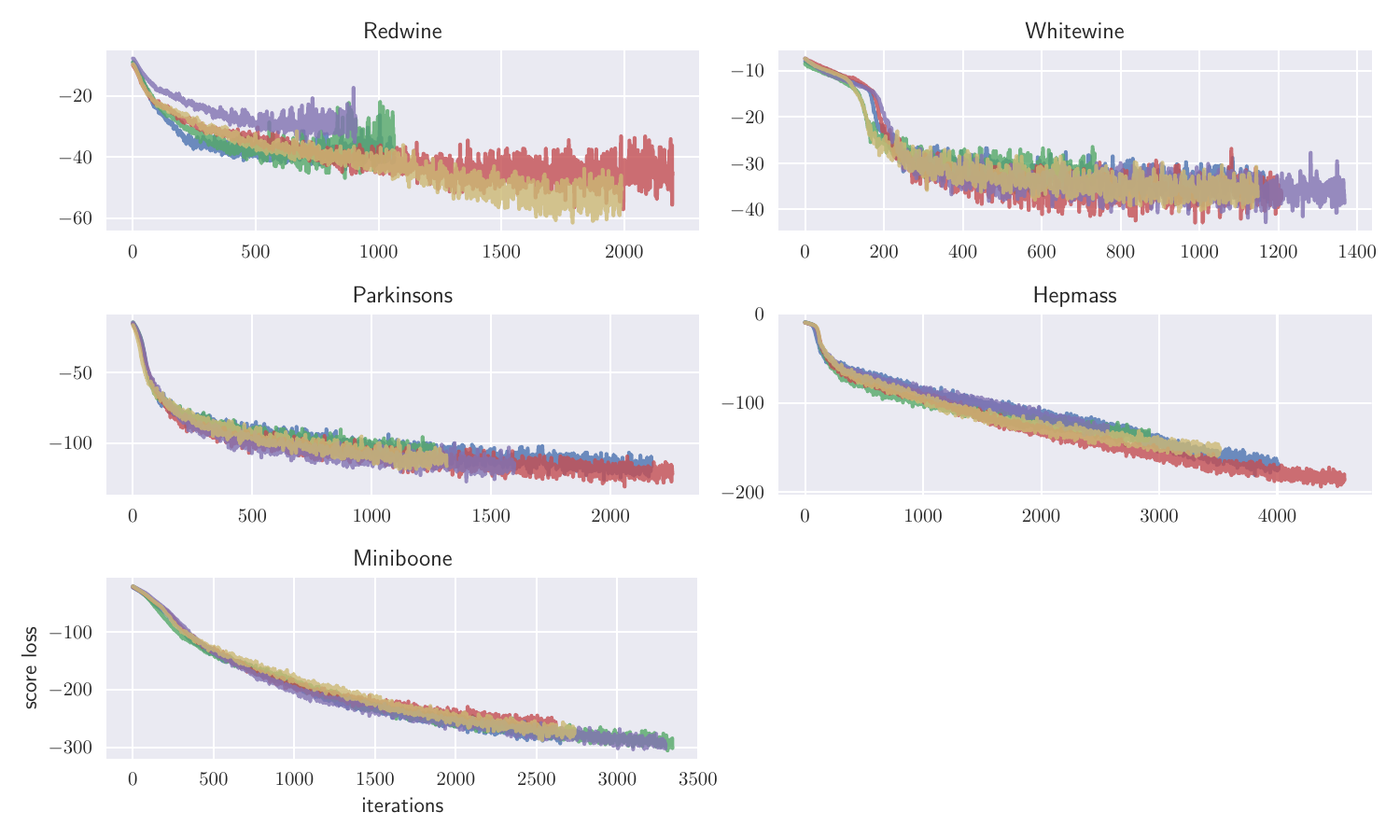}
    \caption{Validation score loss on 5 benchmark datasets for 5 runs.}
    \label{fig:real_traces}
\end{figure}

\subsection{Benchmark datasets}\label{sec:real-data-details}

\paragraph{Pre-processing}
RedWine and WhiteWine are quantized, and thus problematic for modeling with continuous densities;
we added to each dimension uniform noise with support equal to the median distances between two adjacent values.
For HepMass and MiniBoone, we removed ill-conditioned dimensions
as did \citet{maf}. For all datasets except HepMass, 10\% of the entire data was used as testing, and 10\% of the remaining was used
for validation with an upper limit of $1\,000$ due to time cost of validation at each iteration.  
For HepMass, we used the same splitting as done in \citet{maf} and with the same upper limit on validation set.
The data is then whitened before fitting and the whitening matrix was computed on 
at most $10\,000$ data points.

\paragraph{Likelihood-based models}
We set MADE, MADE-MOG, each autoregressive layer of MAF and each 
scaling and shifting layers of real NVP to have two hidden layers of 100 neurons. For real NVP, MAF and MAF-MOG,
five autoregressive layers were used; MAF-MOG and MADE-MOG has a mixture of 10 Gaussians for each 
conditional distribution. Learning rate was $10^{-3}$
The size of a minibatch is 200.

\paragraph{Deep kernel exponential family}
We set the DKEF model to have three kernels ($R=3$),
each a Gaussian on features of a 3-layer network with 30 neurons in each layer. There was also a 
skip-layer connection from data directly to the last layer which accelerated learning. Length scales $\sigma_r$ 
were initialized to 1.0, 3.3 and 10.0. Each $\lambda$ was initialized to 0.001. The 
weights of the network were initialized from a Gaussian distribution with 
standard deviation equal to $1/\sqrt{30}$. We also optimized the inducing points $\zm$
which were initialized with random draws from training data.
The number of inducing points $M=300$, and $|\Dt|=|\Dv|=100$.
The learning rate was $10^{-2}$. We found that our initialization on the weight std 
and $\sigma_r$'s are importance for fast and stable learning; 
other parameters did not significantly change the results under 
similar computational budget (time and memory).

FSSD tests were conducted using 100 points $\vv_b$ selected at random from the test set,
with added normal noise of standard deviation $0.2$,
using code provided by the authors. 

We estimated $\log Z_\vtheta$ with $10^{10}$ samples proposed from $q_0$,
as in \cref{sec:model-eval},
and estimated the bias as in \cref{sec:bias-estimator}.

We added independent $\N(0, 0.05^2)$ noise to the data in training.
This is similar to the regularization applied by \citep{KinLec10,deen},
except that the noise is added directly to the data instead of the model.

For all models, we stopped training when the objective (\eqref{eq:score-estimator} or log likelihood) did not improve for 200 minibatches. 
We also set a time budget of 3 hours on each model; this was fully spent
by MAF, MOG-MAF and Real NVP on HepMass.
We found that MOG-MADE had unstable runs on some datasets;
out of 15 runs on each dataset, 
7 on WhiteWine, 4 on Parkinsons and 9 on MiniBoone produced invalid log likelihoods.
These results were discarded in \cref{fig:real_data} log likelihood panels.

The DKEF in our main results (\cref{fig:real_data}) has an adaptive $q_0$ which is a generalized normal distribution. 
We also trained DKEF with $q_0$ being an isotropic multivariate normal of standard deviation 2.0. These results \cref{fig:real_data_q0} are similar to \cref{fig:real_data} but exhibit much smaller bias estimates in the log normalizer
for RedWine and Parkinsons.

\begin{figure}
    \centering
    \includegraphics[width=1.0\textwidth]{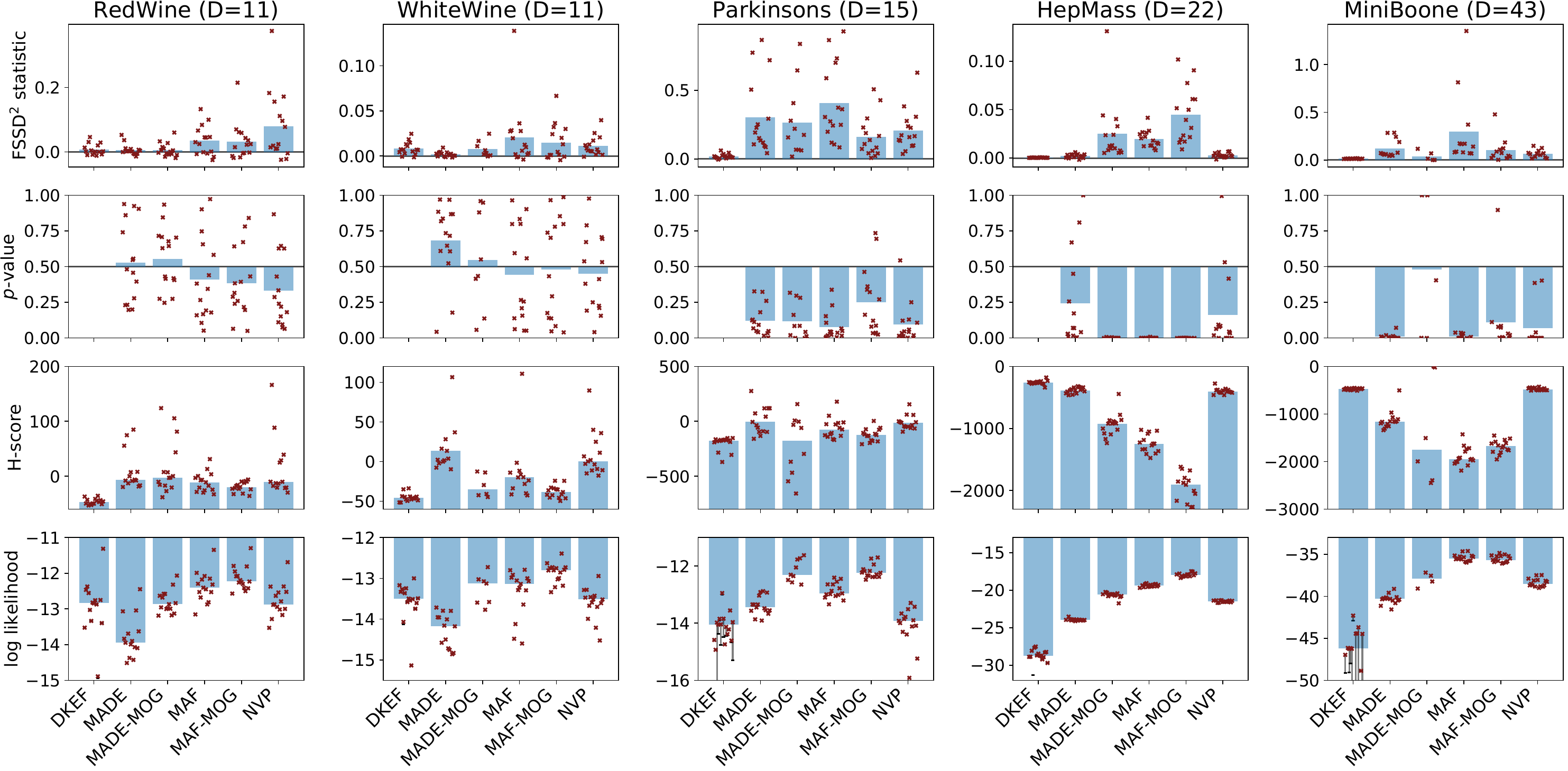}
    \caption{Results on benchmark datasets as in \cref{fig:real_data} with the $q_0$ in DKEF being isotropic multivariate normal of std 2.0.}
    \label{fig:real_data_q0}
\end{figure}

\end{document}